\documentclass{article}

\usepackage{microtype}
\usepackage{graphicx}
\usepackage{subcaption}
\usepackage{booktabs} 
\graphicspath{{./}}
\usepackage{multirow}
\usepackage{algorithmic}
\usepackage{tabularx}
\usepackage{makecell}
\usepackage{float}
\usepackage{threeparttable}
\newcolumntype{Y}{>{\raggedleft\arraybackslash}X} 

\usepackage{xcolor} 
\usepackage{hyperref}

\usepackage[preprint]{icml2026}

\usepackage{amsmath}
\usepackage{amssymb}
\usepackage{mathtools}
\usepackage{amsthm}

\usepackage[capitalize,noabbrev]{cleveref}

\theoremstyle{plain}
\newtheorem{theorem}{Theorem}[section]

\theoremstyle{definition}

\theoremstyle{remark}

\usepackage[textsize=tiny]{todonotes}

\newcommand{\pmstd}[2]{#1\,\scriptsize\,\ensuremath{\pm}\,#2}
\newcommand{\best}[1]{\textbf{#1}}
\newcommand{\second}[1]{\underline{#1}}
\newcommand{\rotfam}[1]{\rotatebox[origin=c]{90}{\scriptsize\bfseries #1}}
\newcommand{\cmark}{\ensuremath{\textcolor{green!60!black}{{\checkmark}}}}
\newcommand{\xmark}{\ensuremath{\textcolor{red!70!black}{{\times}}}}

\icmltitlerunning{Multi-Agent Debate: A Unified Agentic Framework for Tabular Anomaly Detection}

\begin{document}

\twocolumn[
  \icmltitle{Multi-Agent Debate: A Unified Agentic Framework for Tabular Anomaly Detection}



  \icmlsetsymbol{equal}{*}

  \begin{icmlauthorlist}
    \icmlauthor{Pinqiao Wang}{uva}
    \icmlauthor{Sheng Li}{uva}
  \end{icmlauthorlist}

  \icmlaffiliation{uva}{University of Virginia, Charlottesville, USA}

  \icmlcorrespondingauthor{Sheng Li}{shengli@virginia.edu}

  \icmlkeywords{Machine Learning, Agentic AI, Artificial Intelligence, Tabular Data Analysis, Anomaly Detection}

  \vskip 0.3in
]



\printAffiliationsAndNotice{}  

\begin{abstract}
Tabular anomaly detection is often handled by single detectors or static ensembles, even though strong performance on tabular data typically comes from heterogeneous model families (e.g., tree ensembles, deep tabular networks, and tabular foundation models) that frequently disagree under distribution shift, missingness, and rare-anomaly regimes. We propose \textbf{MAD}, a Multi-Agent Debating framework that treats this disagreement as a first-class signal and resolves it through a mathematically grounded coordination layer. Each agent is a machine learning (ML)-based detector that produces a normalized anomaly score, confidence, and structured evidence, augmented by a large language model (LLM)-based critic. A coordinator converts these messages into bounded per-agent losses and updates agent influence via an exponentiated-gradient rule, yielding both a final debated anomaly score and an auditable debate trace. 
MAD is a unified agentic framework that can recover existing approaches, such as mixture-of-experts gating and learning-with-expert-advice aggregation, by restricting the message space and synthesis operator. 
We establish regret guarantees for the synthesized losses and show how conformal calibration can wrap the debated score to control false positives under exchangeability. Experiments on diverse tabular anomaly benchmarks show improved robustness over baselines and clearer traces of model disagreement.
\end{abstract}

\section{Introduction}
Tabular anomaly detection underpins high-stakes monitoring in finance, healthcare, security, and operations, yet it remains fragile in practice. The challenge is not only the heterogeneity of tabular features, but also the setting itself: anomalies are rare, labels are limited or noisy, and deployment routinely faces missingness and distribution shift \citep{han2022adbench}. As a result, no single detector family is consistently reliable. Existing methods, such as tree ensembles, deep tabular models, and tabular foundation models, can each dominate on different datasets, and they often disagree on the same example \citep{chen2016xgboost,ke2017lightgbm,prokhorenkova2018catboost,grinsztajn2022trees,tabpfn,tabpfnv2}.


We argue that this disagreement is not a nuisance to be averaged away; it is a \emph{signal} about ambiguity and failure modes. Leveraging ensemble/committee disagreement as a signal has been explored in other anomaly/novelty detection settings (e.g., semi-supervised novelty detection and medical image anomaly detection) and in contextual anomaly detection via committee-based disagreement measures \citep{tifrea2022ssnd,gu2024d2ue,calikus2022wiscon}. Standard practice in \emph{tabular} anomaly detection, however, still largely either (i) picks one model and hopes it generalizes, or (ii) uses static ensembling that smooths away conflicts without explaining them. This is especially problematic for rare-event metrics and shift-heavy benchmarks, where the right action is often to \emph{resolve} conflicts rather than suppress them.


Inspired by this observation, our idea is simple: when detectors disagree, we should not just average their scores. Instead, we let each detector justify its score (with a calibrated confidence and optional evidence), and a coordinator learns to trust some detectors more than others, especially in the contentious cases where disagreement is high. The result is both a single anomaly score and a transparent trace that records how disagreements were resolved; we introduce the formal agents-and-updates mechanism later. To ground this, we conduct a preliminary study across heterogeneous model families and quantify inter-model disagreement using rank-normalized anomaly scores. Disagreement patterns differ substantially by family, and the benefit of dispute-aware coordination depends on the disagreement regime. These observations motivate a coordination layer that explicitly models ``who disagrees with whom and why,'' rather than treating disagreement as noise to be averaged away.

\begin{figure*}[t]
  \centering
  \includegraphics[width=\textwidth, trim=20 50 20 40, clip]{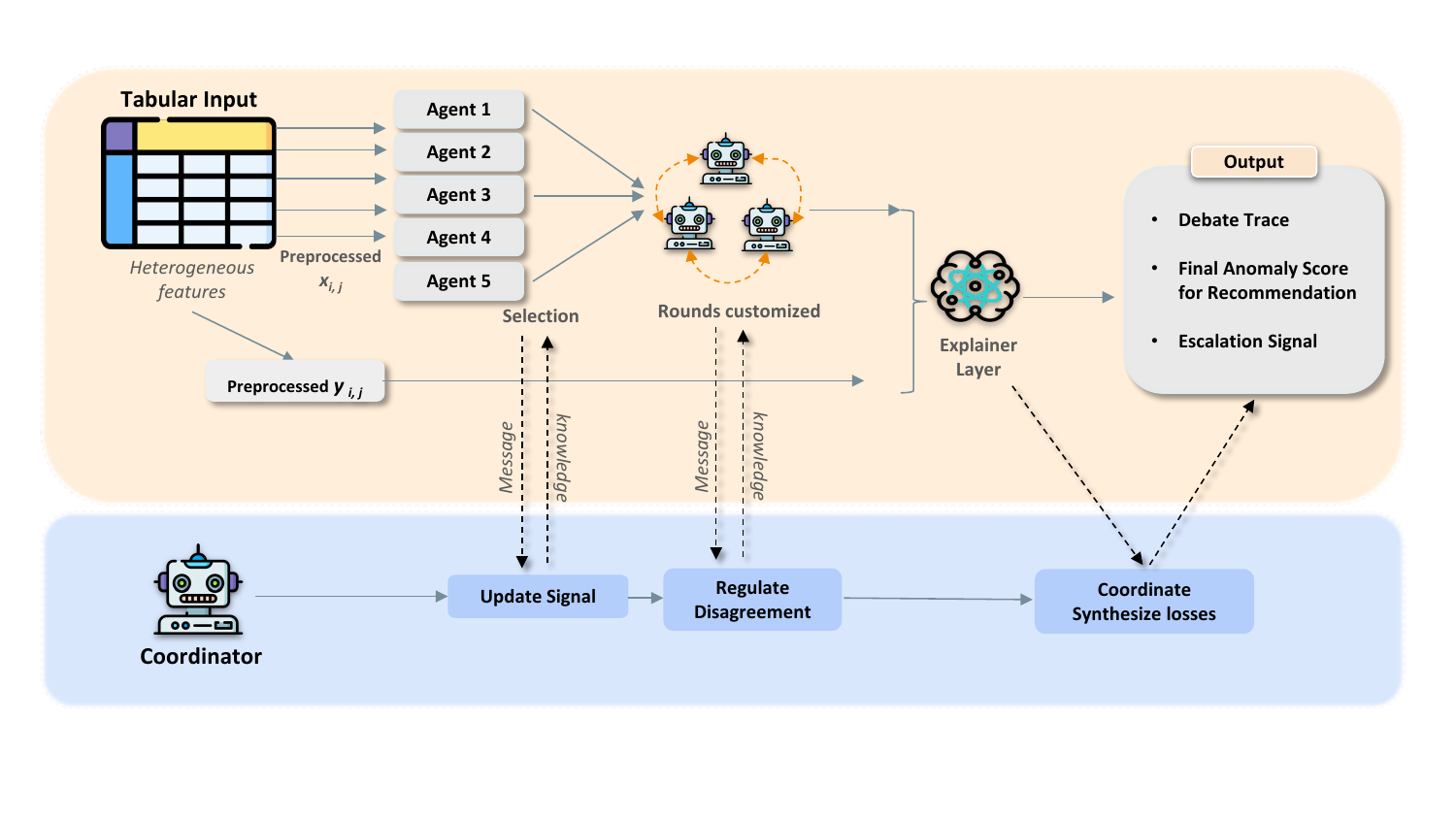}
  \caption{\textbf{MAD Design Overview}
can be decomposed into four modular building blocks: Perception(From input to agent selection), Action(agent debate), Coordinator, and Output. It separates signal extraction, agent interaction, coordination logic, and decision reporting.}
  \label{fig:mad-design}
\end{figure*}


We therefore propose \textbf{MAD}, a multi-agent debating framework for tabular anomaly detection, as shown in Figure~\ref{fig:mad-design}. MAD contains multiple detectors, and each detector is an \emph{ML-based agent} that emits a normalized score along with confidence and optional structured evidence (e.g., feature attributions or counterfactual cues). An optional LLM critic is used only to verify evidence consistency. A coordinator then compares these messages, down-weights detectors whose high-confidence claims are poorly supported, and produces both a final anomaly score and an auditable debate trace. Conceptually, MAD instantiates a broader \emph{superset} view: by restricting the message space and the coordination rule, the same framework recovers standard ensembles (e.g., bagging/stacking) \citep{breiman1996bagging,wolpert1992stacking}, mixture-of-experts-style weighting \citep{jacobs1991moe,shazeer2017moe}, and learning-with-expert-advice aggregation \citep{freund1997decision,cesabianchi2006prediction,shalevshwartz2012online} as special cases. Importantly, the novelty lies in the message schema and in how disagreement and evidence are converted into actionable feedback for coordination.

The main contributions of this work are as follows:
\vspace{-2mm}
\begin{itemize}\itemsep0em
\item \textbf{A Unified Agentic Framework.} Inspired by the observations on structured disagreement from heterogeneous tabular detectors, we design a unifed agentic framework for tabular anomaly detection, MAD, which leverages multiple agents and a coordinator to infer anomaly scores and provide debate trace for better interpretability.  

\item \textbf{Theoretical Justifications.} We formalize debate-augmented anomaly detection as an instance of expert aggregation with disagreement, showing that standard ensembles arise as a degenerate single-round case. We further provide a unifying framework that connects classical learning-theoretic guarantees for expert advice to multi-round agent interactions, clarifying when and why debate improves robustness beyond averaging.

\item \textbf{Extensive Evaluations and Case Studies.} We evaluate MAD on diverse tabular anomaly detection benchmarks and show consistent improvements over strong classical, deep, AutoML, and unsupervised baselines. We further present trace-based case studies that demonstrate how agent disagreement yields interpretable evidence supporting calibrated, human-in-the-loop decisions.
  
  
\end{itemize}

\section{Related Work}
\noindent\textbf{Tabular learning and model families.}
Tree ensembles are commonly strong tabular baselines due to robustness and favorable bias--variance tradeoffs \citep{chen2016xgboost,ke2017lightgbm,prokhorenkova2018catboost,grinsztajn2022trees}. Deep tabular models narrow the gap by adding inductive bias and improved training protocols, including attentive feature selection, differentiable tree-like layers, and Transformer variants \citep{arik2021tabnet,popov2020node,huang2020tabtransformer,gorishniy2021rtdl}, with recent retrieval- and ensembling-style improvements \citep{gorishniy2024tabr,gorishniy2025tabm}. In parallel, prior-data fitted networks and tabular foundation models perform amortized inference over task distributions and can be competitive in small-data regimes \citep{hollmann2023tabpfn,hollmann2025nature,tabpfn,tabpfnv2,liu2025tabpfnunleashed}. Our work does not advocate a single family; it treats heterogeneity (and the disagreements it induces) as an input to coordination.

\noindent\textbf{Tabular anomaly detection and explanation.}
Classic unsupervised detectors include Isolation Forest, LOF, and one-class methods \citep{liu2008iforest,breunig2000lof,scholkopf2001ocsvm}, alongside deep objectives such as Deep SVDD and density--reconstruction hybrids \citep{ruff2018deepsvdd,zong2018dagmm}. Benchmarks highlight large cross-dataset variance and sensitivity to corruption and shift, motivating standardized evaluation \citep{han2022adbench}. Because anomaly decisions are operationally consequential, practitioners commonly attach post hoc explanations (e.g., SHAP/LIME) and counterfactual-style diagnostics \citep{lundberg2017shap,ribeiro2016lime,wachter2017counterfactual}. MAD, by design, makes disagreement and justification part of the decision process.

\noindent\textbf{From aggregation to deliberation.}
Model combination has a long history in bagging/boosting and in online learning with expert advice \citep{cesabianchi2006prediction,shalevshwartz2012online}, and modern systems also use conditional weighting via mixture-of-experts \citep{shazeer2017moe}. Separately, LLM research explores deliberation via search, reflection, and debate protocols \citep{yao2023treeofthoughts,shinn2023reflexion,du2023debate}. MAD sits between these threads: it keeps the \emph{predictive core} as heterogeneous ML detectors, but introduces a debate-style \emph{dispute-resolution layer} with bounded feedback so that coordination is auditable and compatible with standard aggregation guarantees.

\section{Preliminaries and Positioning}\label{sec:prelim}


Tabular anomaly detection is rarely dominated by a single model family: tree-based detectors, deep tabular models, and tabular foundation models can each win on different datasets and can disagree sharply under distribution shift, missingness, and rare anomaly types \citep{han2022adbench,liu2025tabpfnunleashed}.
Averaging scores often hides these disagreements rather than resolving them.
Our goal is a \emph{dispute-resolution layer} that (i) combines heterogeneous ML agents with clean guarantees, and
(ii) exposes \emph{why} a final decision was made. A consolidated list of notations and definitions used throughout this paper is provided in Appendix~\ref{sec:appendix_notation}.

\noindent\textbf{Backbone: learning with expert advice.}
We adopt the prediction-with-expert-advice framework
\citep{freund1997decision,cesabianchi2006prediction,shalevshwartz2012online,hazan2016oco}.
At each round $t$ we maintain weights $\alpha^{(t)}\in\Delta^{N-1}$ over $N$ experts (here: anomaly agents).
Given a bounded loss vector $\ell^{(t)}\in[0,1]^N$, we update the weights using exponentiated gradient
(multiplicative weights) \citep{kivinen1997eg,arora2012mw}:
\begin{equation}\label{eq:eg_prelim}
\alpha_i^{(t+1)}=\frac{\alpha_i^{(t)}\exp(-\eta\,\ell_i^{(t)})}{\sum_{j=1}^N \alpha_j^{(t)}\exp(-\eta\,\ell_j^{(t)})}.
\end{equation}
This backbone is intentionally conservative: it is widely understood, easy to implement, and comes with a regret
bound (over the provided loss sequence). The central design question becomes: \emph{how should we define the per-agent losses $\ell_i^{(t)}$ when
agents disagree?} MAD answers this by turning disagreement and evidence into bounded losses.

\begin{theorem}[Hedge/EG regret]\label{thm:hedge_regret}
Assume $\ell_i^{(t)}\in[0,1]$ for all $i,t$ and initialize $\alpha^{(1)}$ uniformly.
Let $\alpha^{(t)}$ be updated by \eqref{eq:eg_prelim} with $\eta\in(0,1]$. Then
\begin{equation}\label{eq:hedge_regret}
\sum_{t=1}^T\langle \alpha^{(t)},\ell^{(t)}\rangle-\min_{i\in[N]}\sum_{t=1}^T \ell_i^{(t)}
\le \frac{\log N}{\eta}+\frac{\eta T}{8}.
\end{equation}
\end{theorem}
\noindent The complete derivation is standard and included in Appendix~\ref{sec:appendix_proofs}.

\medskip
\noindent\textbf{Superset viewpoint.}
Many existing aggregation schemes can be described as: agents emit outputs; a coordinator combines them.
Our framework makes this explicit by introducing a \emph{message space} and a \emph{loss-synthesis operator}.
At each round, a coordinator receives messages $\{m_i^{(t)}\}_{i=1}^N$ and produces a bounded loss vector
$\widehat{\ell}^{(t)}\in[0,1]^N$:
\begin{equation}\label{eq:psi_prelim_main}
(\widehat{\ell}^{(t)},k^{(t+1)})=\Psi\!\left(k^{(t)},\{m_i^{(t)}\}_{i=1}^N,x_t\right).
\end{equation}
This loss vector is then fed into EG \eqref{eq:eg_prelim}. When messages are restricted to scalar predictions and $\Psi$
ignores interaction, this recovers classical one-shot ensembles and expert aggregation; when messages include structured
evidence and $\Psi$ encodes dispute resolution, we obtain debate-augmented aggregation. A formal embedding of common
frameworks and the full ``superset tuple'' are provided in Appendix~\ref{sec:appendix_framework}.

\medskip
\noindent\textbf{What is different from existing aggregation frameworks.}
Most existing aggregation methods can be summarized as ``combine predictions,'' differing mainly in \emph{when} weights
are chosen and \emph{what information} is used. MAD changes the interface: it does not only aggregate scalar outputs,
but introduces a message $m_i^{(t)}$ and a loss-synthesis step $\Psi$ that converts disagreement and evidence into
bounded feedback. This creates three differences.

\begin{table}
\centering
\small
\label{tab:conceptual-summary}
\vspace{-2pt}
\[
\setlength{\arraycolsep}{7pt}
\renewcommand{\arraystretch}{1.10}
\begin{array}{l|c|c}
\hline
\textbf{Framework} & \textbf{Interact?} & \textbf{Learned?} \\
\hline
\text{Bagging / RF} & \xmark & \xmark \\
\text{Boosting} & \xmark & \xmark \\
\text{Mixture of Experts} & \xmark & \xmark \\
\text{LLM debate} & \cmark & \xmark \\
\text{MAD (ours)} & \cmark & \cmark \\
\hline
\end{array}
\]
\caption{Major Differences with Other Frameworks.}
\end{table}

\noindent\textbf{(i) From prediction-only to message-based coordination.}
Classical ensembles typically consume only a scalar score from each model.
MAD consumes $m_i^{(t)}=(\tilde s_i^{(t)},c_i^{(t)},e_i^{(t)})$, so the coordinator can penalize not only inaccurate
scores but also unjustified or inconsistent evidence.

\smallskip
\noindent\textbf{(ii) From fixed or black-box weighting to auditable dispute resolution.}
In many ensembles, the final score is $\hat s=\sum_i \alpha_i \tilde s_i$ with weights chosen by heuristics, training,
or performance history, but the reason for a particular weighting is not explicit.
MAD produces a bounded loss vector $\widehat{\ell}^{(t)}$ via $\Psi$ and updates weights with EG; the debate trace
records which disagreement or evidence inconsistency caused an agent to be down-weighted.

\smallskip
\noindent\textbf{(iii) Superset view with strict special cases.}
If we restrict the message space to scalar scores ($\mathcal{M}=[0,1]$) and let $\Psi$ ignore interaction, MAD reduces
to standard ensembles or classical expert advice. MAD is therefore a strict superset: it contains existing methods as
special cases and extends them by allowing structured communication. If false-positive control is required, conformal calibration can wrap any final score into a p-value under exchangeability \citep{vovk2005alrw}. We apply this to MAD as an optional post-processing step; details are in Appendix~\ref{sec:appendix_conformal}.

\section{MAD for Tabular Anomaly Detection}\label{sec:method}
\begin{algorithm}[t]
\caption{MAD: Multi-Agent Debate with EG Aggregation}\label{alg:mad_main}
\begin{algorithmic}[1]
\REQUIRE Agents $\{s_i,\nu_i\}_{i=1}^N$, rounds $T$, stepsize $\eta$, synthesis operator $\Psi$
\STATE Initialize weights $\alpha^{(1)} \leftarrow (1/N,\dots,1/N)$
\FOR{each input $x$ (or each round $t$ in a stream)}
  \STATE Initialize state $k^{(1)} \leftarrow \alpha^{(1)}$ (plus bookkeeping for trace)
  \FOR{$t=1$ \textbf{to} $T$}
    \FOR{$i=1$ \textbf{to} $N$}
      \STATE Compute normalized score $\tilde s_i^{(t)}(x)\leftarrow \nu_i(s_i(x))$
      \STATE Compute confidence $c_i^{(t)}(x)$ and evidence $e_i^{(t)}(x)$
      \STATE Form message $m_i^{(t)}(x)\leftarrow(\tilde s_i^{(t)}(x),c_i^{(t)}(x),e_i^{(t)}(x))$
    \ENDFOR
    \STATE $(\widehat{\ell}^{(t)},k^{(t+1)})\leftarrow \Psi(k^{(t)},\{m_i^{(t)}(x)\}_{i=1}^N,x)$
    \STATE Update weights $\alpha^{(t+1)}$ using \eqref{eq:mad_eg_main}
  \ENDFOR
  \STATE Output $\hat s(x)\leftarrow \sum_i \alpha_i^{(T)}\tilde s_i^{(T)}(x)$ and store trace
\ENDFOR
\end{algorithmic}
\end{algorithm}

MAD treats each anomaly detector as an agent and turns disagreement into a structured debate.
Each agent produces (i) a score, (ii) a confidence, and (iii) evidence (for example, feature attributions or
counterfactual cues). A coordinator converts these messages into bounded per-agent losses and updates agent weights
using exponentiated gradient (EG). The output is both a final anomaly score and an auditable trace showing which
agents were trusted and why.

\medskip
\noindent\textbf{Agents and normalized scores.}
Agent $i$ outputs a raw anomaly score $s_i:\mathcal{X}\to\mathbb{R}$.
To make scores comparable across heterogeneous families, we use a monotone map $\nu_i:\mathbb{R}\to[0,1]$ and define
\begin{equation}\label{eq:norm_score_main}
\tilde s_i(x) \triangleq \nu_i(s_i(x))\in[0,1].
\end{equation}
This step removes scale effects so that aggregation and disagreement are meaningful across agents.

\medskip
\noindent\textbf{Debate as loss synthesis.}
MAD introduces a synthesis operator
\begin{equation}\label{eq:mad_psi_main}
(\widehat{\ell}^{(t)},k^{(t+1)})=\Psi\!\left(k^{(t)},\{m_i^{(t)}(x)\}_{i=1}^N,x\right),
\widehat{\ell}^{(t)}\in[0,1]^N
\end{equation}
which is where ``debate'' lives. The role of $\Psi$ is to convert disagreement and evidence into bounded feedback that
down-weights unreliable or unjustified agents while keeping the prerequisites required for EG theory.

\medskip
\noindent\textbf{Messages.}
At round $t$, agent $i$ emits a message
\begin{equation}\label{eq:mad_message_main}
m_i^{(t)}(x)\triangleq \big(\tilde s_i^{(t)}(x),\; c_i^{(t)}(x),\; e_i^{(t)}(x)\big)\in\mathcal{M},
\end{equation}
where $c_i^{(t)}(x)\in[0,1]$ is confidence and $e_i^{(t)}(x)$ is structured evidence. Evidence can be produced by ML
explainers (for example feature attributions) and can optionally be augmented by an LLM critic that checks consistency
or highlights conflicts; MAD does not require the LLM to generate the anomaly score.

\medskip
\noindent\textbf{Aggregation and update.}
MAD maintains weights $\alpha^{(t)}\in\Delta^{N-1}$ and forms the debated aggregate score
\begin{equation}\label{eq:mad_score_main}
\hat s^{(t)}(x)\triangleq \sum_{i=1}^N \alpha_i^{(t)}\,\tilde s_i^{(t)}(x).
\end{equation}
Weights are updated using synthesized losses:
\begin{equation}\label{eq:mad_eg_main}
\alpha_i^{(t+1)}
=
\frac{\alpha_i^{(t)}\exp\!\big(-\eta\,\widehat{\ell}_i^{(t)}\big)}
{\sum_{j=1}^N \alpha_j^{(t)}\exp\!\big(-\eta\,\widehat{\ell}_j^{(t)}\big)}.
\end{equation}
MAD outputs the final score $\hat s(x)\triangleq \hat s^{(T)}(x)$ and a trace of the debate history (messages, losses,
and weight evolution). In practice, $T=1$ is a strong default and $T>1$ is evaluated as an explicit extension.

\medskip
\noindent\textbf{Why confidence and evidence are necessary.}
Score disagreement alone cannot distinguish ``a confident but wrong agent'' from ``a cautious agent with weak signal.''
Confidence $c_i^{(t)}$ and evidence $e_i^{(t)}$ allow $\Psi$ to enforce a simple principle: if an agent disagrees
strongly, it should either provide consistent evidence or be penalized. This is what makes MAD a dispute-resolution
mechanism rather than a reweighting heuristic.

\medskip
Because MAD enforces $\widehat{\ell}_i^{(t)}\in[0,1]$, the aggregation step inherits the standard expert-advice regret
bound (Theorem~\ref{thm:hedge_regret}) with respect to the synthesized losses (not necessarily with respect to a downstream detection metric).

\begin{theorem}[MAD regret]\label{thm:mad_regret_main}
Assume $\widehat{\ell}_i^{(t)}\in[0,1]$ for all $i,t$, where $\widehat{\ell}^{(t)}$ is produced by \eqref{eq:mad_psi_main}.
If MAD updates $\alpha^{(t)}$ by \eqref{eq:mad_eg_main}, then
\[
\sum_{t=1}^T\langle \alpha^{(t)},\widehat{\ell}^{(t)}\rangle
-\min_{i\in[N]}\sum_{t=1}^T \widehat{\ell}_i^{(t)}
\le \frac{\log N}{\eta}+\frac{\eta T}{8}.
\]
\end{theorem}
\noindent Full derivations and additional lemmas are in Appendix~\ref{sec:appendix_proofs}.

\begin{table*}[t]
\vskip 0.06in
\caption{
Popular tabular datasets for anomaly and rare-event detection across domains.
Pos.\% denotes minority/positive (anomaly) rate after binarization.
Binary datasets treat the rarer class as an anomaly; multiclass datasets use one-vs-rest (OvR).
}
\label{tab:datasets}
\centering
\scriptsize
\setlength{\tabcolsep}{2.8pt}
\renewcommand{\arraystretch}{0.8}
\begin{tabular*}{\textwidth}{@{\extracolsep{\fill}} c l l l r r l @{}}
\toprule
\textbf{} & \textbf{Dataset} & \textbf{Domain} & \textbf{Task} & \textbf{Rows} & \textbf{Feat.} & \textbf{Pos.\%} \\
\midrule
\multirow{3}{*}{\rotfam{Finance}}
& Credit Card Fraud & Finance / Fraud & Anomaly       & 284{,}807 & 30   & $\approx$0.17\% \\
& IEEE--CIS Fraud                 & Finance / Fraud & Anomaly       & 590K+     & 400+ & $<1\%$ \\
& Give Me Some Credit             & Consumer Risk   & Rare-event    & 150{,}000 & 10   & $\approx$6\% \\
\midrule
\multirow{3}{*}{\rotfam{Health}}
& Mammography          & Healthcare / Biology & Anomaly       & 11{,}183 & 6   & $\approx$2--3\% \\
& Arrhythmia             & Healthcare / Biology & Anomaly (OvR) & 452      & 279 & rare \\
& Thyroid (ANN)     & Healthcare / Biology & Anomaly       & 7{,}200  & 21  & rare \\
\midrule
\multirow{2}{*}{\rotfam{Eco}}
& Shuttle          & Ops / Physics   & Anomaly (OvR) & 58{,}000  & 9   & rare \\
& CoverType                   & Ecology         & Anomaly (OvR) & 581{,}012 & 54  & rare \\
\midrule
\multirow{3}{*}{\rotfam{Cyber}}
& UNSW-NB15                        & Cybersecurity   & Intrusion     & 2.5M      & 49  & rare \\
& NSL-KDD                          & Cybersecurity   & Intrusion     & 125{,}973 & 41  & rare \\
& CIC-IDS2017                      & Cybersecurity   & Intrusion     & 2--3M     & 75+ & rare \\
\midrule
\multirow{2}{*}{\rotfam{Econ}}
& Bank Marketing      & Business / Marketing & Rare-event & 45{,}211  & 16  & $\approx$11\% \\
& Adult Income                     & Socioeconomic        & Rare-event & 48{,}842  & 14  & $\approx$24\% \\
\midrule
\multirow{2}{*}{\rotfam{Edu}}
& Student Performance (POR)  & Education & Rare-event & 649      & 33 & $\approx$15\% \\
& Academic Performance       & Education & Rare-event & 4{,}800  & 37 & $\approx$12\% \\
\bottomrule
\end{tabular*}
\vskip -0.08in
\end{table*}

\begin{table*}
\vskip 0.06in
\caption{Overall performance (mean $\pm$ std across datasets). $\uparrow$ higher is better; $\downarrow$ lower is better.}
\label{tab:accuracy}
\centering
\scriptsize
\setlength{\tabcolsep}{2.6pt}
\renewcommand{\arraystretch}{0.90}
\begin{tabular*}{\textwidth}{@{\extracolsep{\fill}} c l c c c c c c @{}}
\toprule
& \textbf{Model} &
\textbf{PR} $\uparrow$ &
\textbf{R@1\%FPR} $\uparrow$ &
\textbf{ROC} $\uparrow$ &
\textbf{F1} $\uparrow$ &
\textbf{ECE} $\downarrow$ &
\textbf{Gap} $\downarrow$ \\
\midrule

\multirow{6}{*}{\rotfam{Classical}}
& GLM (LogReg) & \pmstd{0.232}{0.112} & \pmstd{0.401}{0.147} & \pmstd{0.861}{0.043} & \pmstd{0.817}{0.058} & \pmstd{8.3}{3.9} & \pmstd{0.104}{0.061} \\
& RF          & \pmstd{0.267}{0.091} & \pmstd{0.442}{0.118} & \pmstd{0.878}{0.036} & \pmstd{0.834}{0.047} & \pmstd{7.1}{2.8} & \pmstd{0.091}{0.047} \\
& XGB         & \second{\pmstd{0.304}{0.082}} & \second{\pmstd{0.486}{0.105}} & \second{\pmstd{0.899}{0.028}} & \second{\pmstd{0.856}{0.038}} & \pmstd{6.4}{2.2} & \pmstd{0.079}{0.039} \\
& LGBM        & \pmstd{0.296}{0.087} & \pmstd{0.471}{0.121} & \pmstd{0.895}{0.032} & \pmstd{0.852}{0.044} & \pmstd{6.8}{2.6} & \pmstd{0.086}{0.043} \\
& CatBoost    & \pmstd{0.289}{0.095} & \pmstd{0.462}{0.129} & \pmstd{0.897}{0.031} & \pmstd{0.850}{0.043} & \pmstd{6.0}{2.5} & \pmstd{0.072}{0.035} \\
& HeteroStack & \pmstd{0.312}{0.076} & \pmstd{0.498}{0.097} & \pmstd{0.905}{0.024} & \pmstd{0.862}{0.033} & \pmstd{6.2}{2.1} & \pmstd{0.073}{0.033} \\
\midrule

\multirow{3}{*}{\rotfam{AutoML}}
& AutoGluon & \pmstd{0.291}{0.097} & \pmstd{0.470}{0.133} & \pmstd{0.892}{0.041} & \pmstd{0.848}{0.055} & \pmstd{7.4}{3.1} & \pmstd{0.090}{0.052} \\
& H2O      & \pmstd{0.279}{0.103} & \pmstd{0.458}{0.141} & \pmstd{0.887}{0.045} & \pmstd{0.842}{0.060} & \pmstd{6.9}{3.4} & \pmstd{0.083}{0.050} \\
& auto-skl & \pmstd{0.285}{0.090} & \pmstd{0.463}{0.119} & \pmstd{0.889}{0.038} & \pmstd{0.844}{0.050} & \pmstd{7.8}{3.0} & \pmstd{0.098}{0.056} \\
\midrule

\multirow{5}{*}{\rotfam{Deep}}
& TabNet   & \pmstd{0.261}{0.120} & \pmstd{0.433}{0.156} & \pmstd{0.873}{0.052} & \pmstd{0.828}{0.070} & \pmstd{8.6}{4.2} & \pmstd{0.109}{0.066} \\
& FT-Trans & \pmstd{0.300}{0.089} & \pmstd{0.480}{0.127} & \pmstd{0.897}{0.034} & \pmstd{0.852}{0.049} & \pmstd{6.7}{2.7} & \pmstd{0.081}{0.045} \\
& SAINT    & \pmstd{0.307}{0.084} & \pmstd{0.492}{0.116} & \pmstd{0.901}{0.029} & \pmstd{0.857}{0.041} & \pmstd{6.1}{2.3} & \pmstd{0.078}{0.041} \\
& TabTrans & \pmstd{0.282}{0.101} & \pmstd{0.456}{0.138} & \pmstd{0.891}{0.040} & \pmstd{0.846}{0.058} & \pmstd{7.3}{3.3} & \pmstd{0.096}{0.060} \\
& TabPFN   & \pmstd{0.295}{0.078} & \pmstd{0.489}{0.110} & \pmstd{0.904}{0.022} & \pmstd{0.861}{0.031} & \pmstd{6.9}{2.0} & \pmstd{0.092}{0.040} \\
\midrule

\multirow{4}{*}{\rotfam{OD(sk)}}
& iForest     & \pmstd{0.214}{0.141} & \pmstd{0.354}{0.186} & \pmstd{0.781}{0.067} & -- & -- & \pmstd{0.151}{0.094} \\
& OC-SVM      & \pmstd{0.168}{0.118} & \pmstd{0.302}{0.173} & \pmstd{0.747}{0.074} & -- & -- & \pmstd{0.173}{0.102} \\
& LOF         & \pmstd{0.154}{0.109} & \pmstd{0.286}{0.161} & \pmstd{0.734}{0.079} & -- & -- & \pmstd{0.192}{0.110} \\
& Ellip.Env.  & \pmstd{0.131}{0.096} & \pmstd{0.251}{0.148} & \pmstd{0.711}{0.083} & -- & -- & \pmstd{0.208}{0.121} \\
\midrule

\multirow{5}{*}{\rotfam{OD(py)}}
& HBOS  & \pmstd{0.176}{0.126} & \pmstd{0.319}{0.178} & \pmstd{0.756}{0.072} & -- & -- & \pmstd{0.181}{0.103} \\
& COPOD & \pmstd{0.201}{0.148} & \pmstd{0.343}{0.192} & \pmstd{0.772}{0.068} & -- & -- & \pmstd{0.164}{0.097} \\
& ECOD  & \pmstd{0.208}{0.139} & \pmstd{0.349}{0.185} & \pmstd{0.778}{0.066} & -- & -- & \pmstd{0.158}{0.095} \\
& KNN   & \pmstd{0.161}{0.111} & \pmstd{0.292}{0.163} & \pmstd{0.741}{0.076} & -- & -- & \pmstd{0.195}{0.112} \\
& PCA   & \pmstd{0.147}{0.105} & \pmstd{0.271}{0.156} & \pmstd{0.729}{0.081} & -- & -- & \pmstd{0.203}{0.119} \\
\midrule

\rotfam{MAD} & \best{MAD (ours)} &
\best{\pmstd{0.352}{0.071}} &
\best{\pmstd{0.647}{0.094}} &
\best{\pmstd{0.936}{0.015}} &
\best{\pmstd{0.893}{0.026}} &
\best{\pmstd{5.0}{1.5}} &
\best{\pmstd{0.050}{0.024}} \\
\bottomrule
\end{tabular*}

\vspace{2pt}
{\scriptsize \textit{Abbrev.:} XGB=XGBoost, LGBM=LightGBM, HeteroStack=Heterogeneous Stacking, FT-Trans=FT-Transformer, TabTrans=TabTransformer, iForest=Isolation Forest, OC-SVM=One-Class SVM, Ellip.Env.=Elliptic Envelope, auto-skl=auto-sklearn.}
\vskip -0.08in
\end{table*}
\medskip

Algorithm 1 shows the detailed procedures of MAD with EG aggregation, which treats $\Psi$ as a modular component. In addition, Appendix~\ref{sec:appendix_mad_instantiation} specifies one concrete instantiation of $\Psi$ used in our experiments.

\section{Experiments}
\label{sec:experiments}

We evaluate MAD as a debate-augmented instantiation of its math structure for tabular anomaly and rare-event detection. The experiments answer: (i) whether disagreement-aware synthesis improves detection beyond standard ensembling and strong tabular baselines, (ii) whether disagreement predicts where improvements occur, and (iii) which components are responsible for accuracy, reliability, and robustness. More details can be found in appendix~\ref{sec:appendix_exp_details}

\textbf{Datasets, baselines, and protocol.}
We benchmark across diverse domains and imbalance regimes; the dataset suite is summarized in Table~\ref{tab:datasets}.
We draw datasets from OpenML and UCI (and standard intrusion/fraud benchmarks) following their canonical sources
\citep{vanschoren2014openml,dua2019uci,tavallaee2009nslkdd,moustafa2015unsw,sharafaldin2018cicids,dalpozzolo2015ccfraud,kaggle2019ieeecis,kaggle2011givemecredit}.
We compare MAD against representative tabular paradigms:
classical ML (logistic regression; random forests \citep{breiman2001rf}; gradient boosting \citep{friedman2001gbm} via XGBoost \citep{chen2016xgboost}, LightGBM \citep{ke2017lightgbm}, CatBoost \citep{prokhorenkova2018catboost}),
AutoML (AutoGluon \citep{erickson2020autogluon}, H2O AutoML \citep{leDell2020h2oautoml}, auto-sklearn \citep{feurer2015autosklearn}),
deep tabular models (TabNet \citep{arik2021tabnet}, FT-Transformer \citep{gorishniy2021rtdl}, SAINT \citep{somepalli2021saint}, TabTransformer \citep{huang2020tabtransformer}, TabPFN \citep{hollmann2023tabpfn}),
and unsupervised outlier detection (Isolation Forest \citep{liu2008iforest}, One-Class SVM \citep{scholkopf2001ocsvm}, LOF \citep{breunig2000lof}, robust covariance / Elliptic Envelope \citep{rousseeuw1999fastmcd},
and PyOD methods \citep{zhao2019pyod} including HBOS \citep{goldstein2012hbos}, COPOD \citep{li2020copod}, and ECOD \citep{li2022ecod}).
We implement standard baselines using widely used toolkits (scikit-learn \citep{pedregosa2011sklearn} and PyOD \citep{zhao2019pyod}) with shared preprocessing, splits, and tuning budget.
We report PR-AUC and Recall@1\%FPR for rare-event sensitivity, ROC-AUC and macro-F1 for overall discrimination, and ECE \citep{guo2017calibration} plus a slice gap statistic for reliability and robustness. The complete data-preprocessing, split identifiers, and any filtering are provided in the
released code/configs (Appendix~\ref{sec:appendix_code}).

\subsection*{Main Results and Analysis}
\label{subsec:main-results}

\begin{figure*}[t]
\vskip 0.03in
\centering
\includegraphics[width=\textwidth, trim=6 6 6 6, clip]{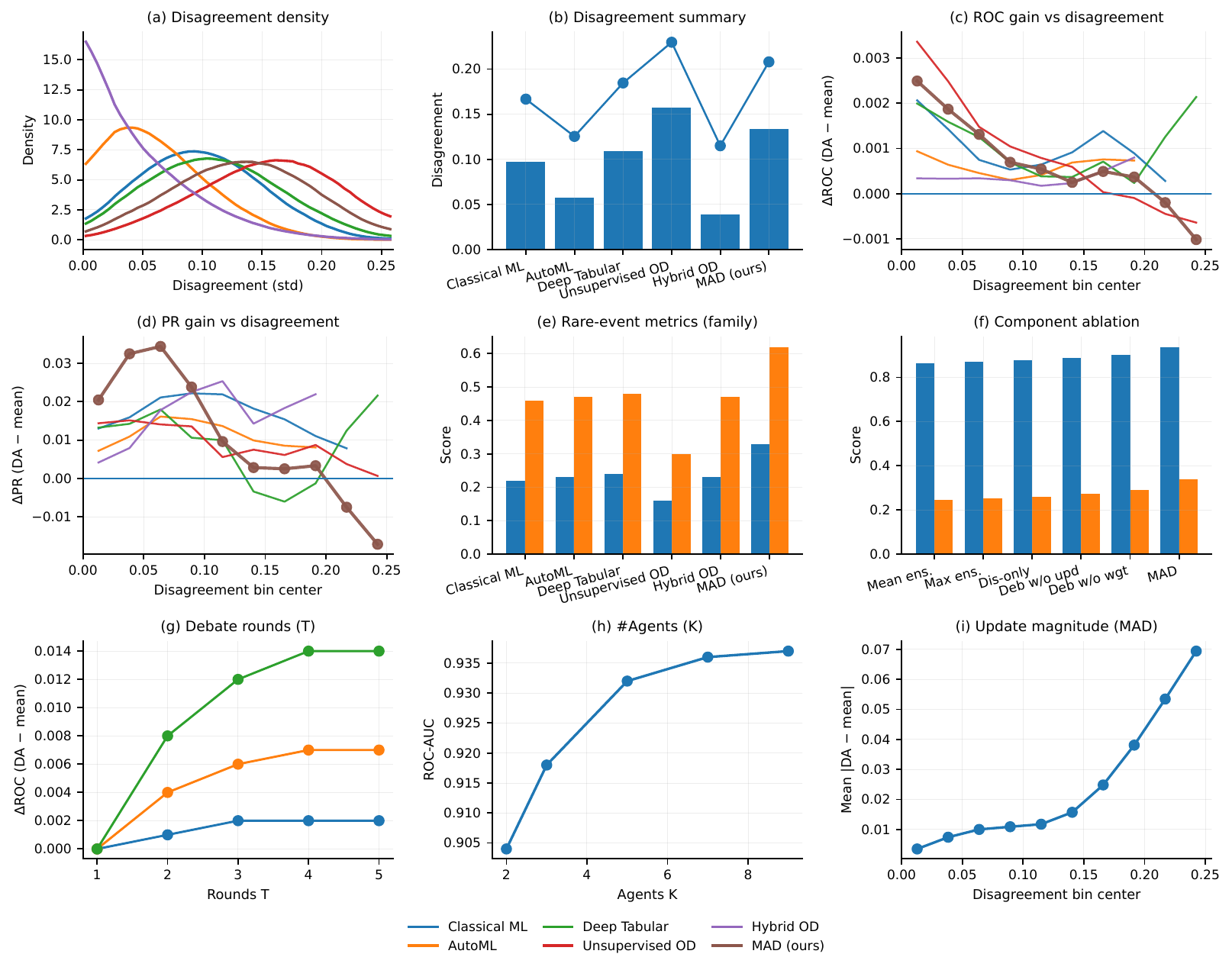}
\caption{\textbf{Results and diagnostics in one view.}
(a) Disagreement density by family. (b) Disagreement summary (median with 90th percentile). (c) $\Delta$ROC (dispute-aware minus mean ensemble) vs.\ disagreement. (d) $\Delta$PR vs.\ disagreement.
(e) Family-level rare-event metrics (PR-AUC, Recall@1\%FPR). (f) Component ablation. (g) Effect of debate rounds $T$ (low/mid/high disagreement). (h) Effect of agent pool size $K$. (i) Mean update magnitude $|\text{DA}-\text{mean}|$ vs.\ disagreement (MAD).}
\label{fig:summary-panels}
\vskip -0.10in
\end{figure*}

\noindent\textbf{Accuracy, calibration, and robustness.}
Table~\ref{tab:accuracy} shows that MAD is consistently competitive at the top across detection metrics while also improving reliability (ECE) and reducing slice disparity (Gap), rather than trading one objective for another. The same conclusion appears in the family aggregation view in Figure~\ref{fig:summary-panels}(e): MAD yields the strongest rare-event operating behavior, especially under strict false-positive budgets where practical screening systems operate. Together with the dataset diversity in Table~\ref{tab:datasets}, this supports that MAD’s gains are not confined to a single domain or a single model family.

\noindent\textbf{Disagreement Diagnostics.}
Figure~\ref{fig:summary-panels}(a,b) shows that disagreement is structured and differs substantially across families, motivating a synthesis rule that reacts to conflict rather than uniformly averaging it away. Figure~\ref{fig:summary-panels}(c,d) then shows that improvements over mean ensembling concentrate in higher-disagreement regimes, while low-disagreement regions see little change, matching the intended “intervene only when needed” behavior of MAD. Consistently, Figure~\ref{fig:summary-panels}(i) shows that the magnitude of MAD’s adjustment tracks disagreement: the debated score stays close to standard aggregation under consensus and departs more when agents conflict.

\noindent\textbf{Ablations Study.} Figure~\ref{fig:summary-panels}(f) and Table~\ref{tab:abl_core} show that removing key debate ingredients (e.g., dropping the update or weighting mechanism, or collapsing to single-channel signals) degrades performance, indicating that MAD’s gains are not explained by ensembling alone. Table~\ref{tab:abl_evid} further shows that both the confidence design and evidence channel matter: weakening evidence or corrupting it reduces the benefits, which aligns with MAD’s goal of using structured support to resolve disagreement rather than simply amplifying variance. Finally, Figures~\ref{fig:summary-panels}(g,h) indicate that MAD benefits from a small number of debate rounds and a moderate agent pool before saturating, suggesting a practical operating point that balances improvements with overhead.

\begin{table*}[t]
\vskip 0.06in
\caption{Core ablations on synthesis/operator choices (mean $\pm$ std across datasets).}
\label{tab:abl_core}
\centering
\scriptsize
\setlength{\tabcolsep}{2.9pt}
\renewcommand{\arraystretch}{0.92}
\begin{tabular*}{\textwidth}{@{\extracolsep{\fill}} c l c c c c @{}}
\toprule
& \textbf{Variant} & \textbf{PR} $\uparrow$ & \textbf{R@1\%} $\uparrow$ & \textbf{ECE} $\downarrow$ & \textbf{Gap} $\downarrow$ \\
\midrule
\multirow{4}{*}{\rotfam{$\nu_i$}}
& Rank/pctl (def.)          & \pmstd{0.352}{0.071} & \pmstd{0.647}{0.094} & \pmstd{5.0}{1.5} & \pmstd{0.050}{0.024} \\
& Min--max                  & \pmstd{0.331}{0.083} & \pmstd{0.616}{0.116} & \pmstd{6.1}{2.1} & \pmstd{0.066}{0.034} \\
& Z $\rightarrow$ sigm.     & \pmstd{0.317}{0.097} & \pmstd{0.589}{0.132} & \pmstd{6.8}{2.7} & \pmstd{0.074}{0.040} \\
& Iso (val-fit)             & \pmstd{0.344}{0.076} & \pmstd{0.635}{0.101} & \pmstd{4.6}{1.4} & \pmstd{0.048}{0.023} \\
\midrule
\multirow{3}{*}{\rotfam{Bound}}
& Clip (def.)               & \pmstd{0.352}{0.071} & \pmstd{0.647}{0.094} & \pmstd{5.0}{1.5} & \pmstd{0.050}{0.024} \\
& Sigmoid link              & \pmstd{0.346}{0.080} & \pmstd{0.636}{0.108} & \pmstd{4.8}{1.6} & \pmstd{0.054}{0.028} \\
& Tanh-rescale              & \pmstd{0.340}{0.086} & \pmstd{0.624}{0.118} & \pmstd{5.3}{1.9} & \pmstd{0.058}{0.030} \\
\midrule
\multirow{5}{*}{\rotfam{Ch.}}
& Pred-only ($\lambda_d{=}0,\lambda_e{=}0$) & \pmstd{0.301}{0.110} & \pmstd{0.563}{0.152} & \pmstd{7.5}{3.6} & \pmstd{0.102}{0.061} \\
& Pred+disp ($\lambda_e{=}0$)               & \pmstd{0.332}{0.094} & \pmstd{0.607}{0.132} & \pmstd{6.3}{2.8} & \pmstd{0.071}{0.045} \\
& Pred+evid ($\lambda_d{=}0$)               & \pmstd{0.323}{0.101} & \pmstd{0.596}{0.144} & \pmstd{5.8}{2.4} & \pmstd{0.079}{0.049} \\
& Disp-only                                 & \pmstd{0.255}{0.121} & \pmstd{0.482}{0.171} & \pmstd{8.7}{4.1} & \pmstd{0.141}{0.083} \\
& Evid-only                                 & \pmstd{0.271}{0.118} & \pmstd{0.507}{0.166} & \pmstd{8.0}{3.7} & \pmstd{0.126}{0.076} \\
\midrule
\rotfam{MAD} & \best{Full (def.)} &
\best{\pmstd{0.352}{0.071}} &
\best{\pmstd{0.647}{0.094}} &
\best{\pmstd{5.0}{1.5}} &
\best{\pmstd{0.050}{0.024}} \\
\bottomrule
\end{tabular*}

\vspace{2pt}
{\scriptsize \textit{Legend:} Bound=bounding/link; Ch.=channels; disp=disagreement; evid=evidence; pctl=percentile; sigm.=sigmoid; Iso=isotonic.}
\vskip -0.08in
\end{table*}

\begin{table*}[t]
\vskip 0.06in
\caption{Ablations on confidence $c_i$ and evidence $e_i$ (mean $\pm$ std across datasets).}
\label{tab:abl_evid}
\centering
\scriptsize
\setlength{\tabcolsep}{2.9pt}
\renewcommand{\arraystretch}{0.92}
\begin{tabular*}{\textwidth}{@{\extracolsep{\fill}} c l c c c c @{}}
\toprule
& \textbf{Variant} & \textbf{PR} $\uparrow$ & \textbf{R@1\%} $\uparrow$ & \textbf{ECE} $\downarrow$ & \textbf{Gap} $\downarrow$ \\
\midrule
\multirow{4}{*}{\rotfam{$c_i$}}
& BootVar (def.)              & \pmstd{0.352}{0.071} & \pmstd{0.647}{0.094} & \pmstd{5.0}{1.5} & \pmstd{0.050}{0.024} \\
& Margin/prob proxy           & \pmstd{0.338}{0.082} & \pmstd{0.621}{0.112} & \pmstd{5.9}{2.1} & \pmstd{0.062}{0.032} \\
& Temp-calibrated             & \pmstd{0.347}{0.078} & \pmstd{0.636}{0.105} & \pmstd{4.4}{1.3} & \pmstd{0.049}{0.024} \\
& Conformal proxy             & \pmstd{0.343}{0.086} & \pmstd{0.628}{0.121} & \pmstd{4.6}{1.5} & \pmstd{0.052}{0.027} \\
\midrule
\multirow{4}{*}{\rotfam{$e_i$}}
& No evidence ($\lambda_e{=}0$) & \pmstd{0.332}{0.094} & \pmstd{0.607}{0.132} & \pmstd{6.3}{2.8} & \pmstd{0.071}{0.045} \\
& $\ell_2$-norm (def.)          & \pmstd{0.352}{0.071} & \pmstd{0.647}{0.094} & \pmstd{5.0}{1.5} & \pmstd{0.050}{0.024} \\
& Top-$k$ sparse ($k{=}10$)     & \pmstd{0.349}{0.077} & \pmstd{0.641}{0.106} & \pmstd{5.2}{1.8} & \pmstd{0.053}{0.027} \\
& Sign-only                     & \pmstd{0.333}{0.091} & \pmstd{0.612}{0.129} & \pmstd{6.0}{2.4} & \pmstd{0.066}{0.039} \\
\midrule
\multirow{3}{*}{\rotfam{Corrupt}}
& 10\% shuffle                 & \pmstd{0.340}{0.089} & \pmstd{0.622}{0.125} & \pmstd{5.8}{2.3} & \pmstd{0.064}{0.038} \\
& 30\% shuffle                 & \pmstd{0.322}{0.103} & \pmstd{0.591}{0.148} & \pmstd{6.6}{2.9} & \pmstd{0.079}{0.048} \\
& 50\% shuffle                 & \pmstd{0.298}{0.118} & \pmstd{0.548}{0.173} & \pmstd{7.9}{3.6} & \pmstd{0.101}{0.061} \\
\midrule
\rotfam{MAD} & \best{Full (def.)} &
\best{\pmstd{0.352}{0.071}} &
\best{\pmstd{0.647}{0.094}} &
\best{\pmstd{5.0}{1.5}} &
\best{\pmstd{0.050}{0.024}} \\
\bottomrule
\end{tabular*}

\vspace{2pt}
{\scriptsize \textit{Legend:} BootVar=bootstrap variance; Temp=temperature; Corrupt=shuffle top-$k$ evidence dimensions.}
\vskip -0.08in
\end{table*}

\subsection*{Operational Implications and Trace-Based Case Study}\label{subsec:casestudy}

\textbf{Operational implications.}
Real-world tabular anomaly detection systems operate under constraints beyond average detection accuracy, including alert budgets, analyst review capacity, and the need to justify decisions under distributional shifts. MAD is designed with these constraints in mind. By producing a ranked anomaly score $\hat s(x)$, MAD directly supports budgeted triage (e.g., reviewing the top-$k$ most suspicious cases). Its online reweighting mechanism reallocates trust across agents when disagreement or instability arises, mitigating the brittleness of static ensembles. Finally, MAD produces an explicit decision trace, exposing how agent disagreement was resolved, which is essential for human-in-the-loop workflows.

\textbf{Trace-based explanation case study.}
We now illustrate how MAD resolves confident disagreement using a representative case study on a single input $x$. The complete trace is provided in Appendix~\ref{sec:appendix_trace}; here we summarize the key reasoning steps.

Consider three agents participating in MAD:
(i) Agent~A, a tree-based detector that historically performs well on this dataset;
(ii) Agent~B, a neural tabular detector with moderate confidence and expressive features;
(iii) Agent~C, a distance-based detector that is conservative but stable.

At debate round $t=1$, the agents emit the following messages:

\begin{table}[H]
\centering
\small
\caption{Agent messages at debate round $t=1$ for the representative case study input $x$.}
\label{tab:case_messages_t1}
\begin{tabular}{lccc}
\toprule
Agent & $\tilde s_i(x)$ & $c_i(x)$ & Evidence consistency \\
\midrule
A & 0.91 & 0.88 & low \\
B & 0.64 & 0.55 & high \\
C & 0.22 & 0.30 & medium \\
\bottomrule
\end{tabular}
\end{table}

As shown in Table~\ref{tab:case_messages_t1}, Agent~A predicts a strong anomaly with high confidence, but its feature attribution sharply disagrees with the consensus attribution formed by the weighted agent pool. Agent~B predicts a moderate anomaly with supporting evidence that aligns well with the consensus, while Agent~C weakly predicts normal behavior with low confidence.

The MAD synthesis operator $\Psi$ assigns a larger loss to Agent~A due to confident disagreement without evidential support, a smaller loss to Agent~B due to evidence-aligned disagreement, and an intermediate loss to Agent~C reflecting uncertainty rather than error. After the EG update, MAD shifts trust away from Agent~A toward Agent~B, producing a final anomaly score that reflects coordinated resolution rather than dominance by a single agent.

Based on the final score $\hat s(x)$ and the associated trace, MAD flags this input as \emph{moderately anomalous}. Rather than triggering an automatic hard action, MAD recommends this case for \emph{human review}, so the anomaly signal is interpreted in light of corroborating evidence from multiple agents instead of a single model’s confidence. The debate trace highlights where agents agree and where uncertainty remains, providing reviewers with a structured summary of the underlying reasoning. This aligns with real deployment practice, in which ambiguous cases are escalated for inspection, enabling experts to apply domain knowledge and balance risk before taking action, while preserving automation for clearly benign or clearly anomalous instances.

\section{Limitations and Failure Cases}
MAD introduces an interaction layer atop standard detectors and thus incurs additional inference cost relative to single models or one-shot ensembles; this cost scales with the number of agents and debate rounds, though we observe diminishing returns once debate stabilizes. Its gains are largest when agents are genuinely diverse: highly correlated agents yield limited disagreement, reducing MAD to behavior similar to conventional ensembling. MAD also depends on the quality of message channels—while disagreement is robust, confidence and evidence can be noisy (e.g., unstable attributions or correlated features), and ablations show that weakening these channels degrades performance (Tables~\ref{tab:abl_core}–\ref{tab:abl_evid}). Finally, debate traces are informative but not self-explanatory: summarization and judgment are still required, and persistent disagreement may reflect bias or underspecification rather than true uncertainty. 

\section{Conclusion}
\label{sec:conclusion}
We study tabular anomaly and rare-event detection through the lens that disagreement is information. While disagreement has been explored in other novelty-detection settings, tabular methods still rely largely on single detectors or static ensembling, where averaging can obscure consistent weak signals and amplify overconfident errors. We introduce MAD, a debate-augmented multi-agent framework in which heterogeneous detectors exchange normalized scores, confidence, and evidence, and a coordinator performs disagreement-aware synthesis to produce both a final anomaly score and an auditable trace. MAD strictly generalizes common aggregation paradigms and empirically improves detection, calibration, and slice robustness across diverse domains, with ablations confirming that gains arise from synthesis design rather than ensembling alone. For future works, we will include making debate adaptive to cost and difficulty, automatically designing and pruning agent portfolios, and learning calibrated message channels. We also see opportunities in risk control under distribution shift and in compressing debate traces into actionable rationales for downstream workflows.

\newpage
\section*{Impact Statement}
Tabular anomaly and rare-event detection guides resource allocation in high-stakes settings (e.g., fraud, cybersecurity, healthcare, and operations). However, improved detectors can also cause harm: false positives may trigger unnecessary investigations or automated denials; performance may differ across subgroups due to historical bias or domain shift; and anomaly scoring can be misused for intrusive surveillance. These risks motivate careful thresholding, human-in-the-loop review for consequential actions, clear use-case documentation, and routine audits for subgroup robustness and calibration under deployment distributions.

MAD does not require sensitive attributes, but we encourage slice evaluations when ethically appropriate and further work on risk control (e.g., conformal calibration under shift), privacy-preserving deployment, and trace-to-rationale interfaces that are concise and actionable for operators.

\bibliography{actual_paper_completed}
\bibliographystyle{icml2026}

\newpage

\appendix
\onecolumn

\setlength{\textfloatsep}{10pt plus 2pt minus 2pt}
\setlength{\floatsep}{8pt plus 2pt minus 2pt}
\setlength{\intextsep}{8pt plus 2pt minus 2pt}
\setlength{\abovedisplayskip}{6pt plus 2pt minus 2pt}
\setlength{\belowdisplayskip}{6pt plus 2pt minus 2pt}
\setlength{\abovedisplayshortskip}{4pt plus 2pt minus 2pt}
\setlength{\belowdisplayshortskip}{4pt plus 2pt minus 2pt}
\renewcommand{\arraystretch}{1.05}
\setlength{\tabcolsep}{5pt}

\section{Use of LLM}
We used an LLM only to polish the writing of the paper.

\section{Experiment Details}\label{sec:appendix_exp_details}
\medskip
\noindent\textbf{Classification-to-anomaly transformation.}
\textbf{Binary datasets:} treat the positive class as anomalies and the negative class as normals.
\textbf{Multi-class datasets:} one-vs-rest evaluation: for each class $c$, treat class $c$ as anomalies and all other
classes as normals; report averages across classes. This yields a controlled ``new anomaly type'' shift by changing $c$.

\medskip
\noindent\textbf{Splits and supervision protocols.}
We use dataset-level train/validation/test splits with fixed random seeds.
\textbf{Supervised protocol:} train agents with available labels; validate hyperparameters; test on held-out data.
\textbf{Semi-supervised protocol:} train using normals only (anomalies withheld); validate on a mix or on synthetic
anomalies; test on true anomalies. We report which protocol applies to each benchmark.

\medskip
\noindent\textbf{Preprocessing.}
We apply a shared preprocessing pipeline across all agents to ensure comparisons reflect modeling, not preprocessing.
Numerical features are standardized using training statistics. Categorical features are encoded consistently
(one-hot or ordinal, depending on dataset constraints) with train-only fitting. Missing values are handled by either
(i) model-native missing handling (for tree models), or (ii) imputation fit on training data (median for numerical,
mode for categorical) plus a missingness indicator when applicable. We avoid leakage by fitting all transforms on the
training split only.

\medskip
\noindent\textbf{Agent pool.}
MAD agents are heterogeneous ML models spanning representative families. In the released implementation, each agent
conforms to a minimal interface:
\[
\texttt{fit(X\_train, y\_train?)},\quad \texttt{score(X)\,$\to$\,raw scores},\quad
\texttt{explain(X)\,$\to$\,attributions (optional)}.
\]
We include a mix of: tree-based detectors, distance/density-based detectors, kernel one-class methods, and neural
tabular detectors. When a model lacks an intrinsic explanation method, we wrap it with a model-agnostic explainer to
produce $a_i(x)$.

\medskip
\noindent\textbf{Confidence definitions.}
Confidence $c_i(x)$ is model-dependent but normalized to $[0,1]$.
Examples include: inverse score variance under bootstraps; margin-based confidence for classifiers; or calibrated
probabilities when available. We provide exact confidence mappings per agent in the code/configs.

\medskip
\noindent\textbf{MAD hyperparameters.}
We tune $\eta$ (EG stepsize), $\lambda$ (dispute strength), $\gamma$ (evidence strength), and optionally $T$
(number of debate rounds). We treat $T=1$ as a default and evaluate $T>1$ as an ablation. We tune using validation sets
and report the chosen values (or simple defaults) per benchmark.

\medskip
\noindent\textbf{Baselines.}
We include: best single agent; uniform averaging; stacking/linear meta-model on validation scores; mixture-of-experts
gating (learned $\alpha(x)$ without debate); and a MAD ablation with $\lambda=\gamma=0$ (prediction-only EG), which
isolates the contribution of dispute resolution.

\medskip
\noindent\textbf{Hyperparameter search and seeds (used for all tables/figures).}
Unless otherwise specified, each dataset is evaluated over a fixed set of random seeds, and we report mean $\pm$ std
across seeds and across datasets (Table~\ref{tab:accuracy}, Tables~\ref{tab:abl_core}--\ref{tab:abl_evid}, and
Figure~\ref{fig:summary-panels}). Hyperparameters are tuned on the validation split with a shared budget per model
family (same number of trials) to avoid favoring any baseline. For MAD we tune $\eta\in\{0.25,0.5,1.0\}$,
$\lambda\in\{0,0.25,0.5,1.0\}$, $\gamma\in\{0,0.25,0.5,1.0\}$, and $T\in\{1,2,3\}$ for the ablation in
Figure~\ref{fig:summary-panels}(g); we select by validation PR-AUC with ECE as a tiebreaker.

\medskip
\noindent\textbf{How disagreement plots are computed.}
All disagreement statistics in Figure~\ref{fig:summary-panels} are computed from rank/quantile-normalized scores
$\tilde s_i(x)$ (Eq.~\eqref{eq:norm_score_main}). At the sample level, we measure disagreement as the variance of
$\{\tilde s_i(x)\}_{i=1}^N$ (or equivalently the mean squared deviation from the mean ensemble score). Dataset-level
disagreement is summarized by the median and 90th percentile over test samples; these summaries correspond to
panels (a,b), and are used as the x-axis in panels (c,d,i).

\subsection{Metrics and Aggregation Across Datasets}\label{sec:appendix_metrics}

\noindent\textbf{Detection and screening metrics.}
We report the same metrics as in Section~\ref{sec:experiments} and Table~\ref{tab:accuracy}: \textbf{(i) PR-AUC}
(AUPRC) and \textbf{(ii) ROC-AUC} for ranking quality; \textbf{(iii) Recall@1\%FPR} (true-positive rate when the
false-positive rate is constrained to 1\%) to reflect strict alert budgets; and \textbf{(iv) macro-F1} when a decision
threshold is required.

\medskip
\noindent\textbf{Calibration and slice robustness.}
\textbf{ECE:} expected calibration error computed on binned predicted probabilities/scores after per-model score
normalization (Eq.~\eqref{eq:norm_score_main}); unless otherwise stated we use equal-width bins on $[0,1]$.
\textbf{Gap:} a slice-disparity statistic computed by partitioning the test set into predefined slices (e.g., missingness
patterns, quantiles of a key feature, or domain-specific subgroups when available) and reporting the max--min gap of the
chosen detection metric across slices. (The exact slice definitions are listed in the experiment configs.)

\medskip
\noindent\textbf{Stability and robustness.}
We report mean and standard deviation across random seeds. For corruption experiments, we report performance curves
versus corruption strength and summarize by area-under-degradation or by worst-case degradation in a fixed range.

\medskip
\noindent\textbf{Across-dataset aggregation.}
We report (i) mean and median performance across datasets, and (ii) win/tie/loss counts versus key baselines.
When appropriate, we use paired comparisons per dataset and summarize with a sign test or Wilcoxon signed-rank test.
(Exact statistical tests used are specified in the code and printed with results.)

\medskip
\noindent\textbf{Efficiency.}
We report wall-clock time per dataset and per sample (inference), and the incremental overhead of debate relative to
one-shot aggregation. For the optional LLM critic, we report the additional latency and number of tokens consumed.

\subsection{Shift and Corruption Suite}\label{sec:appendix_shifts}

\noindent\textbf{Why these shifts.}
Tabular models often fail under feature noise, missingness, scaling drift, and category perturbations. We therefore
evaluate robustness under a controlled corruption suite to measure stability beyond average IID performance.

\medskip
\noindent\textbf{Corruptions.}
We apply corruptions at test time only:
\textbf{(i) Gaussian noise} on numerical columns with scale proportional to training standard deviation;
\textbf{(ii) missing-value injection} at a specified rate per feature;
\textbf{(iii) scaling drift} by multiplying selected numerical features by a drift factor;
\textbf{(iv) categorical perturbation} by swapping categories to an ``unknown'' bucket or by random relabeling under a
fixed budget. We report results across multiple severity levels.

\medskip
\noindent\textbf{New anomaly type shift (one-vs-rest).}
For multi-class datasets, we calibrate on anomalies from one target class and evaluate on a different target class,
to test whether dispute resolution helps when the anomaly semantics change.

\subsection{Full MAD Trace for Case Study Input}\label{sec:appendix_trace}

\noindent\textbf{Initial state.}
The initial agent weights are uniform:
\[
\alpha^{(1)} = (0.33,\;0.33,\;0.34).
\]

\medskip
\noindent\textbf{Round $t=1$: Agent messages.}
\[
\begin{aligned}
m_A^{(1)} &= (\tilde s_A(x)=0.91,\; c_A(x)=0.88,\; a_A(x)), \\
m_B^{(1)} &= (\tilde s_B(x)=0.64,\; c_B(x)=0.55,\; a_B(x)), \\
m_C^{(1)} &= (\tilde s_C(x)=0.22,\; c_C(x)=0.30,\; a_C(x)).
\end{aligned}
\]

The consensus attribution $\bar a^{(1)}(x)$ aligns closely with $a_B(x)$, partially with $a_C(x)$, and poorly with $a_A(x)$.

\medskip
\noindent\textbf{Synthesized losses.}
Applying $\Psi$ yields the bounded loss vector:
\[
\widehat{\ell}^{(1)} = (0.42,\;0.18,\;0.26).
\]
Agent~A is penalized for confident disagreement without evidential support, while Agent~B is rewarded for evidence-consistent disagreement.

\medskip
\noindent\textbf{EG update.}
With learning rate $\eta=1.0$, the EG update produces:
\[
\alpha^{(2)} = (0.22,\;0.46,\;0.32).
\]

\medskip
\noindent\textbf{Final anomaly score.}
Using the updated weights, MAD outputs:
\[
\hat s(x) = 0.22 \cdot 0.91 + 0.46 \cdot 0.64 + 0.32 \cdot 0.22 = 0.51.
\]

\medskip
\noindent\textbf{Interpretation.}
Although a single agent predicted a strong anomaly, MAD discounted this signal due to a lack of supporting evidence. The final score reflects a consensus-driven assessment, balancing disagreement, confidence, and evidence consistency.

\medskip
\noindent\textbf{Decision guidance.}
The trace suggests that this case should be escalated for review rather than automatically blocked or ignored. In practice, such trace-level explanations allow operators to understand \emph{why} an alert was raised and which agents influenced the outcome, supporting accountable and trustworthy anomaly detection decisions.

\subsection{Code Pointers and Flagship Functions}\label{sec:appendix_code}

\noindent\textbf{Design goal.}
The implementation is intentionally small: MAD is a coordination layer that sits on top of a pool of agent models.
Experiments call a stable set of functions that mirror the math in the paper.

\medskip
\noindent\textbf{Flagship functions (mapping math $\leftrightarrow$ code).}
The following names are the reference API used throughout experiments (exact file paths depend on your repo layout):
\begin{verbatim}
# Core MAD loop (Section 4 + Alg. 1 in main body)
mad_forward(x, agents, alpha, state) -> messages
synthesize_losses(x, messages, alpha, params) -> loss_vec, new_state    # Psi
eg_update(alpha, loss_vec, eta) -> alpha_next

# Scoring + tracing
run_mad_on_dataset(dataset, agents, params, T, seeds) -> scores, traces

# Evaluation (Section 5)
evaluate_detection(scores, labels, metrics) -> result_dict
run_corruption_suite(dataset, corruption_spec, ...) -> curves
aggregate_across_datasets(results_list) -> summary_tables

# Agents
build_agent_pool(config) -> agents
score_agent(agent, X) -> raw_scores
explain_agent(agent, X) -> attributions (optional)
\end{verbatim}

\medskip
\noindent\textbf{How Section 5 connects to code.}
\textbf{Datasets and splits} correspond to the benchmark specification and loaders used by
\texttt{run\_mad\_on\_dataset}. \textbf{Agents} correspond to \texttt{build\_agent\_pool} and per-agent wrappers.
\textbf{Debate parameters} ($\eta,\lambda,\gamma,T$) correspond to the \texttt{params} passed to
\texttt{synthesize\_losses} and the loop length. \textbf{Shifts/corruptions} correspond to
\texttt{run\_corruption\_suite}. \textbf{Tables/figures} are produced by \texttt{aggregate\_across\_datasets} and plotting
utilities in the experiment scripts.

\section{Framework Breakdown and Embeddings}\label{sec:appendix_framework}

\noindent\textbf{Full superset tuple.}
We define a multi-agent learning-and-synthesis system as the tuple
\begin{equation}\label{eq:masls_tuple_app}
\mathcal{A}=
\Big(
\{\mathcal{H}_i\}_{i=1}^N,\;
\mathcal{M},\;
\{\pi_i\}_{i=1}^N,\;
\{U_i\}_{i=1}^N,\;
F,\;
\Psi,\;
\mathcal{K},\;
g,\;
T
\Big).
\end{equation}
$\mathcal{H}_i$ is agent $i$'s hypothesis class (model family), $\mathcal{M}$ is the message space, $\pi_i$ maps an input
(and optional context) to a message, $U_i$ is an optional within-debate update rule for agent $i$, $F$ is a shared tabular
interface, $k^{(t)}\in\mathcal{K}$ is the coordinator state at debate round $t$, $\Psi$ is the synthesis operator that
produces bounded losses and updates the coordinator state, and $g$ maps the final state and messages to outputs
(prediction and trace). $T$ is the number of debate rounds per input.

\medskip
\noindent\textbf{Unifying interface (what the coordinator ``sees'').}
At debate round $t$ on input $x$, each agent emits $m_i^{(t)}\in\mathcal{M}$ and the coordinator applies
\begin{equation}\label{eq:psi_appendix}
(\widehat{\ell}^{(t)},k^{(t+1)})=\Psi\!\left(k^{(t)},\{m_i^{(t)}\}_{i=1}^N,x\right),
\qquad \widehat{\ell}^{(t)}\in[0,1]^N.
\end{equation}
The coordinator then updates weights by EG using $\widehat{\ell}^{(t)}$ (main body Eq.~\eqref{eq:eg_prelim} or
Eq.~\eqref{eq:mad_eg_main}, depending on your numbering).

\medskip
\noindent\textbf{Embedding common frameworks as special cases.}
The point of \eqref{eq:masls_tuple_app} is not to introduce new machinery, but to make it explicit that many familiar
pipelines differ only by restricting $\mathcal{M}$ and choosing $\Psi$ and $g$.

\smallskip
\noindent\textbf{(a) Static ensembles.}
Let $\mathcal{M}=[0,1]$ and $m_i^{(1)}(x)=\tilde s_i(x)$ (a scalar score). Choose fixed weights $\alpha$ and set
$g(\cdot)$ to output $\hat s(x)=\sum_i \alpha_i \tilde s_i(x)$. This recovers uniform or weighted averaging.

\smallskip
\noindent\textbf{(b) Classical expert advice (Hedge/EG).}
Let $\mathcal{M}=[0,1]$ and define $\Psi$ to output the task loss vector $\widehat{\ell}^{(t)}=\ell^{(t)}$ (for example,
a supervised loss when labels exist). Then EG updates match the standard Hedge protocol.

\smallskip
\noindent\textbf{(c) Mixture-of-experts (gating).}
Let $\mathcal{M}=[0,1]$ and choose a learned gating network that outputs $\alpha(x)$ directly; set $T=1$, do not update
$\alpha$ online, and output $\hat s(x)=\sum_i \alpha_i(x)\tilde s_i(x)$. This yields an MoE-style predictor.

\smallskip
\noindent\textbf{(d) MAD (ours).}
Let $\mathcal{M}$ include scalar score, confidence, and structured evidence. Choose $\Psi$ to convert score disagreement
and evidence consistency into bounded losses, then apply EG. This yields debate-augmented aggregation with a trace.

\medskip
\noindent\textbf{Why this matters for ``ML vs LLM.''}
The superset view isolates what must be ML (base scoring agents) from what can be LLM (optional critique of evidence).
LLMs appear only through the \emph{message channel} (as part of $e_i^{(t)}$ or as an additional critic agent), not as a
mandatory scorer. This keeps the backbone in standard online learning while allowing modern agentic components.

\subsection{Concrete MAD Synthesis Operator Used in Experiments}\label{sec:appendix_mad_instantiation}

\noindent\textbf{Goal of $\Psi$.}
MAD needs a synthesis operator $\Psi$ that (i) produces bounded losses $\widehat{\ell}^{(t)}\in[0,1]^N$ so that EG theory
applies, and (ii) operationalizes dispute resolution by penalizing disagreement that is not supported by confidence or
evidence.

\medskip
\noindent\textbf{Message schema.}
We use
\[
m_i^{(t)}(x)=\big(\tilde s_i^{(t)}(x),\,c_i^{(t)}(x),\,e_i^{(t)}(x)\big),
\]
where $\tilde s_i^{(t)}(x)\in[0,1]$, $c_i^{(t)}(x)\in[0,1]$, and
\[
e_i^{(t)}(x)=\big(a_i^{(t)}(x),\,\delta_i^{(t)}(x),\,r_i^{(t)}(x)\big).
\]
Here $a_i^{(t)}(x)\in\mathbb{R}^d$ is a feature attribution vector (ML explainer), $\delta_i^{(t)}(x)\in\mathbb{R}^d$
is an optional counterfactual direction, and $r_i^{(t)}(x)$ is an optional natural-language critique/rationale (which may
be produced by an LLM critic). Only $a_i^{(t)}(x)$ is used in the default quantitative evidence term below; $r_i^{(t)}(x)$
is used in trace analysis.

\medskip
\noindent\textbf{Step 1: normalize evidence.}
When $a_i^{(t)}(x)\neq 0$, define $\tilde a_i^{(t)}(x)=a_i^{(t)}(x)/\|a_i^{(t)}(x)\|_2$; otherwise set $\tilde a_i^{(t)}(x)=0$.

\medskip
\noindent\textbf{Step 2: compute a consensus attribution.}
Given current weights $\alpha^{(t)}$,
\begin{equation}\label{eq:consensus_attr_app_full}
\bar a^{(t)}(x)=\sum_{i=1}^N \alpha_i^{(t)}\,\tilde a_i^{(t)}(x),
\qquad
\tilde{\bar a}^{(t)}(x)=
\begin{cases}
\bar a^{(t)}(x)/\|\bar a^{(t)}(x)\|_2,& \bar a^{(t)}(x)\neq 0,\\
0,& \bar a^{(t)}(x)=0.
\end{cases}
\end{equation}

\medskip
\noindent\textbf{Step 3: prediction feedback term.}
When labels $y\in\{0,1\}$ exist for $x$, we use a bounded supervised loss
\begin{equation}\label{eq:pred_loss_sup_app_full}
\ell_{\mathrm{pred},i}^{(t)}(x)
=
-y\log\big(\max\{\tilde s_i^{(t)}(x),\epsilon\}\big)
-(1-y)\log\big(\max\{1-\tilde s_i^{(t)}(x),\epsilon\}\big),
\end{equation}
where $\epsilon>0$ is a small constant to avoid $\log(0)$ when scores saturate at $0$ or $1$.
When labels do not exist, we use a stability proxy with perturbations $x'\sim \mathsf{Pert}(\cdot\mid x)$:
\begin{equation}\label{eq:pred_loss_unsup_app_full}
\ell_{\mathrm{pred},i}^{(t)}(x)
=
\mathbb{E}_{x'\sim \mathsf{Pert}(\cdot\mid x)}
\left[\left|\tilde s_i^{(t)}(x)-\tilde s_i^{(t)}(x')\right|\right].
\end{equation}
In practice we estimate the expectation by $K$ perturbation samples. This proxy penalizes agents whose scores are highly
unstable under small feature perturbations, a common failure mode in tabular settings.

\medskip
\noindent\textbf{Step 4: dispute term (score disagreement).}
Let the current aggregate score be $\hat s^{(t)}(x)=\sum_i \alpha_i^{(t)}\tilde s_i^{(t)}(x)$. Define
\begin{equation}\label{eq:disp_loss_app_full}
\ell_{\mathrm{disp},i}^{(t)}(x)=
c_i^{(t)}(x)\cdot\left(\tilde s_i^{(t)}(x)-\hat s^{(t)}(x)\right)^2.
\end{equation}
This penalizes confident disagreement more than low-confidence disagreement.

\medskip
\noindent\textbf{Step 5: evidence term (agreement with consensus).}
\begin{equation}\label{eq:evid_loss_app_full}
\ell_{\mathrm{evid},i}^{(t)}(x)=
c_i^{(t)}(x)\cdot\Big(1-\cos\!\big(\tilde a_i^{(t)}(x),\tilde{\bar a}^{(t)}(x)\big)\Big),
\end{equation}
with the convention $\ell_{\mathrm{evid},i}^{(t)}(x)=0$ if $\tilde a_i^{(t)}(x)=0$ or $\tilde{\bar a}^{(t)}(x)=0$.
This enforces that strong claims (high $c_i^{(t)}$) should come with evidence that is not arbitrarily inconsistent with
the ensemble's evidence.

\medskip
\noindent\textbf{Step 6: bounded total loss.}
\begin{equation}\label{eq:loss_total_app_full}
\widehat{\ell}_i^{(t)}(x)=
\mathrm{clip}_{[0,1]}\!\left(
\ell_{\mathrm{pred},i}^{(t)}(x)+\lambda\,\ell_{\mathrm{disp},i}^{(t)}(x)+\gamma\,\ell_{\mathrm{evid},i}^{(t)}(x)
\right).
\end{equation}
Clipping ensures $\widehat{\ell}_i^{(t)}(x)\in[0,1]$ and therefore makes the prerequisites of Hedge/EG explicit.

\medskip
\noindent\textbf{Discussion: why this instantiation is reasonable.}
The design separates \emph{task feedback} (prediction term) from \emph{dispute resolution} (dispute + evidence terms).
If $\lambda=\gamma=0$, MAD reduces to expert advice driven purely by task feedback. As $\lambda,\gamma$ increase, the
coordinator becomes more conservative against agents that disagree confidently without providing consistent evidence.
We report ablations over $\lambda,\gamma$ and over the presence/absence of evidence channels.

\section{Proofs of Main Statements}\label{sec:appendix_proofs}

\noindent\textbf{Proof of Theorem~\ref{thm:hedge_regret}.}
We provide a standard proof for completeness. Let $\alpha^{(t)}\in\Delta^{N-1}$ be updated by
\[
\alpha_i^{(t+1)}=
\frac{\alpha_i^{(t)}\exp(-\eta \ell_i^{(t)})}{\sum_{j=1}^N\alpha_j^{(t)}\exp(-\eta \ell_j^{(t)})},
\]
where $\ell_i^{(t)}\in[0,1]$ and $\eta\in(0,1]$. Define the normalizer
\[
Z_t=\sum_{j=1}^N\alpha_j^{(t)}\exp(-\eta \ell_j^{(t)}).
\]
Then
\[
\log Z_t=\log\left(\sum_{j=1}^N\alpha_j^{(t)}\exp(-\eta \ell_j^{(t)})\right).
\]
Using the inequality $\exp(-\eta u)\le 1-\eta u+\eta^2 u^2/2$ for $u\in[0,1]$ and $\eta\in(0,1]$, we obtain
\[
Z_t \le \sum_{j=1}^N \alpha_j^{(t)}\left(1-\eta \ell_j^{(t)}+\frac{\eta^2}{2}(\ell_j^{(t)})^2\right)
=1-\eta\langle \alpha^{(t)},\ell^{(t)}\rangle+\frac{\eta^2}{2}\sum_{j=1}^N\alpha_j^{(t)}(\ell_j^{(t)})^2.
\]
Since $(\ell_j^{(t)})^2\le \ell_j^{(t)}$ for $\ell_j^{(t)}\in[0,1]$, we have
\[
Z_t \le 1-\eta\langle \alpha^{(t)},\ell^{(t)}\rangle+\frac{\eta^2}{2}\langle \alpha^{(t)},\ell^{(t)}\rangle
=1-\eta\left(1-\frac{\eta}{2}\right)\langle \alpha^{(t)},\ell^{(t)}\rangle.
\]
Using $\log(1-u)\le -u$,
\[
\log Z_t \le -\eta\left(1-\frac{\eta}{2}\right)\langle \alpha^{(t)},\ell^{(t)}\rangle.
\]
On the other hand, fix any expert $i^\star$. By iterating the update,
\[
\alpha_{i^\star}^{(T+1)}
=
\alpha_{i^\star}^{(1)}
\frac{\exp\!\left(-\eta\sum_{t=1}^T \ell_{i^\star}^{(t)}\right)}{\prod_{t=1}^T Z_t}.
\]
Taking logs and rearranging gives
\[
\sum_{t=1}^T \log Z_t
=
\log \alpha_{i^\star}^{(1)}
-\log \alpha_{i^\star}^{(T+1)}
-\eta\sum_{t=1}^T \ell_{i^\star}^{(t)}
\ge
\log \alpha_{i^\star}^{(1)}
-\eta\sum_{t=1}^T \ell_{i^\star}^{(t)},
\]
since $\alpha_{i^\star}^{(T+1)}\le 1$. Combining with the upper bound on $\log Z_t$ yields
\[
-\eta\left(1-\frac{\eta}{2}\right)\sum_{t=1}^T \langle \alpha^{(t)},\ell^{(t)}\rangle
\ge
\log \alpha_{i^\star}^{(1)}
-\eta\sum_{t=1}^T \ell_{i^\star}^{(t)}.
\]
Assuming uniform initialization $\alpha_{i^\star}^{(1)}=1/N$ and dividing by $\eta(1-\eta/2)$ gives
\[
\sum_{t=1}^T \langle \alpha^{(t)},\ell^{(t)}\rangle
-
\sum_{t=1}^T \ell_{i^\star}^{(t)}
\le
\frac{\log N}{\eta(1-\eta/2)}.
\]
Using $1/(1-\eta/2)\le 1+\eta$ for $\eta\in(0,1]$ and simplifying yields a standard bound of the form
\[
\sum_{t=1}^T \langle \alpha^{(t)},\ell^{(t)}\rangle
-
\sum_{t=1}^T \ell_{i^\star}^{(t)}
\le
\frac{\log N}{\eta}+\frac{\eta T}{8},
\]
which is Eq.~\eqref{eq:hedge_regret}. \hfill $\square$

\medskip
\noindent\textbf{Proof of Theorem~\ref{thm:mad_regret_main}.}
By construction of MAD, $\widehat{\ell}_i^{(t)}\in[0,1]$ (Eq.~\eqref{eq:loss_total_app_full}).
MAD updates weights by EG using $\ell^{(t)}=\widehat{\ell}^{(t)}$. Therefore Theorem~\ref{thm:hedge_regret} applies
directly, yielding the stated regret bound. \hfill $\square$

\subsection{Conformal Calibration Details}\label{sec:appendix_conformal}

\noindent\textbf{Why conformal is a wrapper (not part of MAD).}
MAD outputs a score $\hat s(x)\in[0,1]$ (or a monotone transform thereof). Conformal prediction can convert any score
into a p-value with finite-sample validity under exchangeability, giving a principled way to set thresholds that control
false positives.

\medskip
\noindent\textbf{Split conformal p-values.}
Let $\mathcal{D}_{\mathrm{cal}}=\{x_1,\dots,x_n\}$ be a calibration set of normal examples (exchangeable with future
normal examples). Let $A(x)=\hat s(x)$ be the MAD score. Define
\[
p(x)=\frac{1+\sum_{j=1}^n \mathbb{I}\{A(x_j)\ge A(x)\}}{n+1}.
\]
The conformal decision rule is $\hat y(x)=\mathbb{I}\{p(x)\le \alpha_{\mathrm{cf}}\}$.

\medskip
\noindent\textbf{Finite-sample guarantee.}
Under exchangeability of $(x_1,\dots,x_n,X)$ when $X$ is a new normal example, the rank of $A(X)$ among
$\{A(x_1),\dots,A(x_n),A(X)\}$ is uniform. Therefore $\mathbb{P}(p(X)\le \alpha_{\mathrm{cf}})\le \alpha_{\mathrm{cf}}$.
This yields false-positive control at level $\alpha_{\mathrm{cf}}$ for normal examples. We report both raw MAD metrics and
conformal-wrapped operating points when appropriate.

\medskip
\noindent\textbf{Practical note.}
When only a subset of normals is available (semi-supervised anomaly detection), we form $\mathcal{D}_{\mathrm{cal}}$ from
held-out normals and compute p-values at evaluation time.

\section{Notation and Symbols}\label{sec:appendix_notation}

\vspace{-0.25em}
\renewcommand{\arraystretch}{1.06}
\setlength{\tabcolsep}{6pt}

\begin{table}[H]
\centering
\caption{Core notation used throughout the paper.}
\vspace{-0.25em}
\begin{tabularx}{0.9\textwidth}{@{} l l X @{}}
\toprule
\textbf{Symbol} & \textbf{Type} & \textbf{Meaning} \\
\midrule
$\mathcal{X}$ & set & Input space of tabular examples. \\
$x$ (or $x_t$) & variable & A tabular example; $x_t$ is the example processed at debate round $t$. \\
$d$ & integer & Number of features (dimensionality of tabular input). \\
$[N]$ & set & $\{1,\dots,N\}$, index set of agents (experts). \\
$N$ & integer & Number of agents in the pool. \\
$T$ & integer & Number of debate rounds per input. \\
$t$ & integer & Debate round index, $t\in\{1,\dots,T\}$. \\
$\Delta^{N-1}$ & set & Probability simplex in $\mathbb{R}^N$ (nonnegative entries summing to $1$). \\
$\alpha^{(t)}$ & vector & Agent weights at round $t$, $\alpha^{(t)}\in\Delta^{N-1}$. \\
$\eta$ & scalar & EG (multiplicative weights) step size. \\
$\langle u,v\rangle$ & op. & Standard inner product. \\
\midrule
$s_i(x)$ & scalar & Raw anomaly score from agent $i$ on input $x$ (scale may vary across agents). \\
$\nu_i(\cdot)$ & function & Monotone map that normalizes $s_i(\cdot)$ into $[0,1]$ (e.g., rank/quantile map). \\
$\tilde s_i(x)$ & scalar & Normalized score: $\tilde s_i(x)=\nu_i(s_i(x))\in[0,1]$. \\
$\hat s^{(t)}(x)$ & scalar & Debated aggregate score at round $t$: $\hat s^{(t)}(x)=\sum_i \alpha_i^{(t)}\tilde s_i^{(t)}(x)$. \\
$\hat s(x)$ & scalar & Final MAD score after $T$ rounds: $\hat s(x)=\hat s^{(T)}(x)$. \\
$c_i^{(t)}(x)$ & scalar & Confidence reported by agent $i$ at round $t$, normalized to $[0,1]$. \\
\bottomrule
\end{tabularx}
\end{table}

\begin{table}
\centering
\caption{Evidence, synthesis instantiation, and optional conformal wrapper.}
\vspace{-0.25em}
\begin{tabularx}{0.9\textwidth}{@{} l l X @{}}
\toprule
\textbf{Symbol} & \textbf{Type} & \textbf{Meaning} \\
\midrule
$a_i^{(t)}(x)$ & vector & Feature attribution vector for agent $i$ at round $t$ (e.g., from an ML explainer), $a_i^{(t)}(x)\in\mathbb{R}^d$. \\
$\tilde a_i^{(t)}(x)$ & vector & $\ell_2$-normalized attribution: $\tilde a_i^{(t)}(x)=a_i^{(t)}(x)/\|a_i^{(t)}(x)\|_2$ (or $0$ if $a_i^{(t)}(x)=0$). \\
$\delta_i^{(t)}(x)$ & vector & Optional counterfactual direction for agent $i$ at round $t$, $\delta_i^{(t)}(x)\in\mathbb{R}^d$. \\
$r_i^{(t)}(x)$ & text & Optional natural-language critique/rationale (may be produced by an LLM critic). \\
$\bar a^{(t)}(x)$ & vector & Weighted consensus attribution: $\bar a^{(t)}(x)=\sum_i \alpha_i^{(t)}\tilde a_i^{(t)}(x)$. \\
$\tilde{\bar a}^{(t)}(x)$ & vector & $\ell_2$-normalized consensus attribution (or $0$ if $\bar a^{(t)}(x)=0$). \\
$\cos(u,v)$ & scalar & Cosine similarity between vectors $u$ and $v$. \\
\midrule
$y$ & label & Ground-truth label when available (e.g., $y\in\{0,1\}$ for normal/anomaly). \\
$\ell_{\mathrm{pred},i}^{(t)}(x)$ & scalar & Prediction-feedback loss term (supervised cross-entropy if $y$ exists; otherwise a stability proxy). \\
$\mathsf{Pert}(\cdot\mid x)$ & dist. & Perturbation distribution around $x$ used for unsupervised stability loss. \\
$K$ & integer & Number of perturbation samples used to estimate the unsupervised stability loss. \\
$\ell_{\mathrm{disp},i}^{(t)}(x)$ & scalar & Dispute (score-disagreement) loss term, typically weighted by confidence. \\
$\ell_{\mathrm{evid},i}^{(t)}(x)$ & scalar & Evidence-consistency loss term (e.g., $1-\cos(\tilde a_i,\tilde{\bar a})$). \\
$\lambda,\gamma$ & scalars & Weights for dispute and evidence terms in the synthesized loss:
$\widehat{\ell}_i^{(t)}=\mathrm{clip}_{[0,1]}\!\left(\ell_{\mathrm{pred},i}^{(t)}+\lambda\ell_{\mathrm{disp},i}^{(t)}+\gamma\ell_{\mathrm{evid},i}^{(t)}\right)$. \\
$\mathrm{clip}_{[0,1]}(\cdot)$ & op. & Clips a scalar to $[0,1]$ to ensure bounded losses. \\
\midrule
$\mathbb{I}\{\cdot\}$ & op. & Indicator function ($1$ if condition holds, else $0$). \\
$\mathcal{D}_{\mathrm{cal}}$ & set & Calibration set for conformal wrapping (typically normal examples). \\
$n$ & integer & Size of calibration set, $n=|\mathcal{D}_{\mathrm{cal}}|$. \\
$A(x)$ & scalar & Nonconformity score used in conformal; here $A(x)=\hat s(x)$. \\
$p(x)$ & scalar & Split conformal p-value:
$p(x)=\frac{1+\sum_{j=1}^n \mathbb{I}\{A(x_j)\ge A(x)\}}{n+1}$. \\
$\alpha_{\mathrm{cf}}$ & scalar & Conformal significance level (distinct from the weight vector $\alpha^{(t)}$). \\
$\hat y(x)$ & label & Conformal decision rule $\hat y(x)=\mathbb{I}\{p(x)\le \alpha_{\mathrm{cf}}\}$. \\
\bottomrule
\end{tabularx}
\end{table}

\begin{table}
\centering
\caption{Superset (MAD) tuple notation (Appendix~\ref{sec:appendix_framework}).}
\vspace{-0.25em}
\begin{tabularx}{0.9\textwidth}{@{} l l X @{}}
\toprule
\textbf{Symbol} & \textbf{Type} & \textbf{Meaning} \\
\midrule
$\mathcal{A}$ & tuple & MAD system tuple:
$\mathcal{A}=(\{\mathcal{H}_i\}_{i=1}^N,\mathcal{M},\{\pi_i\}_{i=1}^N,\{U_i\}_{i=1}^N,F,\Psi,\mathcal{K},g,T)$. \\
$\mathcal{H}_i$ & set & Hypothesis class / model family for agent $i$. \\
$\pi_i$ & function & Agent $i$'s messaging policy mapping $(x,\text{context})$ to a message in $\mathcal{M}$. \\
$U_i$ & update & Optional within-debate update rule for agent $i$ across rounds. \\
$F$ & interface & Shared tabular interface used by agents/coordinator (preprocessing/feature schema/access). \\
$g$ & function & Output mapping from final state/messages to prediction and trace. \\
\bottomrule
\end{tabularx}
\end{table}


\end{document}